\documentclass[11pt]{article}

\usepackage[margin=1in]{geometry}
\usepackage{amsmath,amssymb,amsthm,mathtools}
\usepackage{bm}
\usepackage{comment}
\usepackage{lipsum}
\usepackage[T1]{fontenc}
\usepackage[colorlinks=true,citecolor=blue,linkcolor=blue,urlcolor=blue,backref=page]{hyperref}
\usepackage{graphicx,subfigure}
\usepackage{xcolor}
\usepackage{new_commands}
\usepackage{enumitem}
\usepackage{booktabs}
\usepackage[capitalize,noabbrev]{cleveref}

\theoremstyle{plain}
\newtheorem{theorem}{Theorem}[section]
\newtheorem{example}{Example}[section]
\newtheorem{proposition}[theorem]{Proposition}
\newtheorem{lemma}[theorem]{Lemma}
\newtheorem{corollary}[theorem]{Corollary}
\theoremstyle{definition}
\newtheorem{definition}[theorem]{Definition}

\theoremstyle{remark}
\newtheorem{remark}[theorem]{Remark}

\usepackage[authoryear]{natbib}

\usepackage{blindtext}
\usepackage{titling}
\usepackage{array}
\preauthor{\begin{center}
		\large \lineskip .75em%
		\begin{tabular}[t]{>{\centering\arraybackslash}p{.45\textwidth}}}
		\postauthor{\end{tabular}\par\end{center}}
\makeatletter
\renewcommand\and{%
\end{tabular}%
\hfill
\begin{tabular}[t]{>{\centering\arraybackslash}p{.45\textwidth}}}%
\makeatother

\title{Matrix Estimation for Individual Fairness}

\author{Cindy Y. Zhang$^*$ \\ \texttt{cindyz@princeton.edu}\\ Princeton University  \and Sarah H. Cen$^*$ \\ \texttt{shcen@mit.edu}\\ Massachusetts Institute of Technology \and
Devavrat Shah \\ \texttt{devavrat@mit.edu}\\ Massachusetts Institute of Technology}
\def\thefootnote{*}\footnotetext{Indicates equal contribution}

\date{}

\begin{document}

\maketitle
\begin{abstract}
In recent years, multiple notions of algorithmic fairness have arisen. 
One such notion is individual fairness (IF), which requires that individuals who are similar receive similar treatment.
In parallel, matrix estimation (ME) has emerged as a natural paradigm for handling noisy data with missing values.
In this work, we connect the two concepts. 
We show that pre-processing data using ME can improve an algorithm's IF without sacrificing performance.
Specifically, 
we show that using a popular ME method known as singular value thresholding (SVT) to pre-process the data provides a strong IF guarantee under appropriate conditions. 
We then show that, under analogous conditions, SVT pre-processing also yields estimates that are consistent and approximately minimax optimal.
As such, the ME pre-processing step does not, under the stated conditions, increase the prediction error of the base algorithm, i.e., does not impose a fairness-performance trade-off.
We verify these results on synthetic and real data. 
\end{abstract}
\renewcommand*{\thefootnote}{\arabic{footnote}}

\section{Introduction}
As data-driven decision-making becomes more ubiquitous, there is increasing attention on the \emph{fairness} of machine learning (ML) algorithms.  
 Because what is deemed to be fair is context-dependent (e.g., reflects a given value system), there is no universally accepted notion of fairness. 

One notion of algorithmic fairness
is \emph{individual fairness (IF)}, which is distinct from notions of group fairness (e.g., equalized odds).  
Stated informally, IF says that similar individuals should receive similar treatment. 
More precisely, an algorithm $f: \cX \rightarrow \cY$ acting on a set of individuals $\cX$ is individually fair if for any two individuals $a, b \in \cX,$
\begin{equation}\label{eq:if-gen}
    D(f(a), f(b)) \leq L \cdot d(a,b) ,
\end{equation}
for the choice of distance metrics $d$ and $D$. 
The \emph{Lipschitz constant} $L$ captures how strictly the IF condition is enforced. 
An algorithm $f$ that satisfies IF ensures that the outcomes between two individuals who are close in feature space $\cX$ also receive outcomes that are close in outcome space $\cY$, 
where the level of closeness is captured by $L$. 
A smaller Lipschitz constant therefore  implies a stronger IF constraint.

\begin{figure}
\vskip 0.2in
\begin{center}
\centerline{\includegraphics[width=0.9\columnwidth]{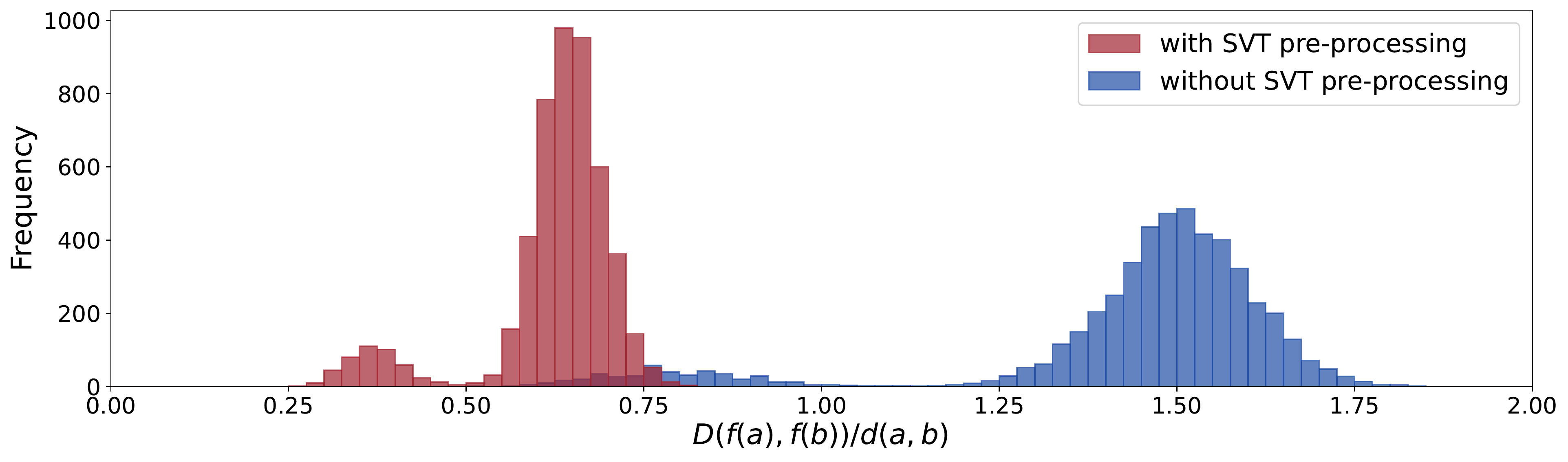}}
\vskip -0.1in
\caption{
We run a deep neural network on synthetic data with and without SVT pre-processing (see Section \ref{sec:experiments}, Experiment \#1 for details). 
We randomly select pairs $a, b \in \cX$ then compute the ratio $D(f(a),f(b))/d(a,b)$, 
where $f$ denotes the neural network with (red) and without (blue) SVT pre-processing.
As shown, 
applying SVT pre-processing results in lower ratios,
which indicates that it improves individual fairness, as defined in \eqref{eq:if-gen}. 
Indeed, 
we show in Section \ref{sec:main_results} that, under appropriate conditions, 
SVT pre-processing strengthens an algorithm's IF guarantee. }
\label{fig:lipschitz-diff}
\end{center}
\vskip -0.25in
\end{figure}

In parallel, 
matrix estimation (ME) has arisen as a natural paradigm to handle data that is noisy and/or has missing values. 
In this work, 
we propose a two-step procedure in which the data (e.g., training data) is first pre-processed using a ME technique known as \emph{singular value thresholding (SVT)}
before being used by an inference algorithm
$h$ (e.g., a neural network). 
We show that, 
under appropriate conditions, 
this pre-processing step \emph{strengthens} the IF guarantee of the inference algorithm, 
i.e., combining SVT with $h$ results in a lower Lipschitz constant in \eqref{eq:if-gen} than $h$ does alone.

Although SVT can improve an algorithm's IF, 
it is not clear whether such an improvement comes at a cost to the algorithm's performance. 
In this work, we show that the same thresholds that allow SVT to improve IF \emph{also} imply that SVT has strong performance guarantees. In other words, under the appropriate conditions,
SVT improves IF without imposing a performance cost in settings where ME can be applied. Our problem setup is visualized in Figure \ref{fig:overview} and described in detail in Section \ref{sec:problem_statement}.

Our main contributions can be summarized as follows:  

\begin{itemize}[leftmargin=*, topsep=0pt]

\item \textbf{We show SVT pre-processing has strong IF guarantees}. 
ME is used in high-dimensional inference to handle sparse, noisy data.
One of the most popular ME methods is SVT. 
In Sections \ref{sec:SVT_IF_Z}-\ref{sec:SVT_IF_A}, we derive a set of conditions under which SVT pre-processing strengthens the IF guarantees of the inference algorithm with respect to the observed covariates and provides an approximate IF guarantee with respect to the (unknown) ground truth covariates.
We then use this result to explore how SVT affects predictions in different data regimes.

\item \textbf{We show that IF under SVT does not hurt asymptotic performance}. 
In Section \ref{sec:USVT}, 
we show that achieving IF using SVT pre-processing does not necessarily hurt performance. 
Specifically, 
we show that the same conditions that are needed for SVT to guarantee IF mirror the conditions required under a popular method known as universal singular value thresholding (USVT). 
Because USVT has strong performance guarantees (it produces an estimator that is consistent and approximately minimax \cite{chatterjee}), 
this connection implies that SVT pre-processing can achieve IF without imposing a performance cost. 
Stated differently, enforcing IF via SVT pre-processing does not harm performance because it places no further restrictions on ME than the performance-based method USVT. 
\begin{figure*}[!htbp]
\begin{center}
\centerline{\includegraphics[width=\textwidth]{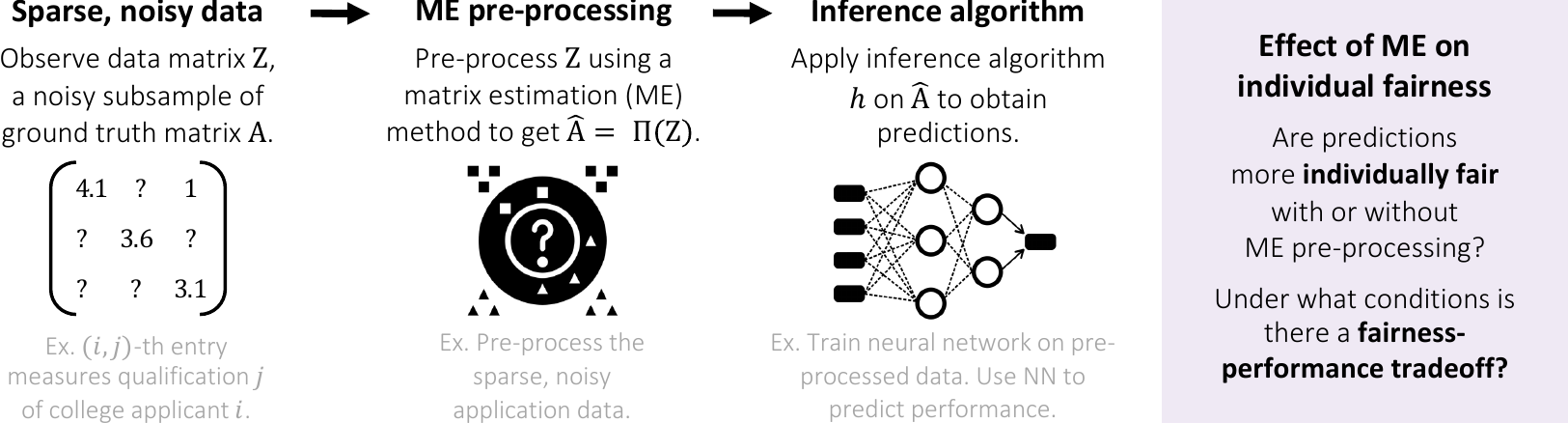}}
\vskip -0.05in
\caption{
We study the effect of ME pre-processing on IF and performance in settings where we need to perform an inference task using sparse, noisy data. Our main results show that SVT, a popular ME method, provides strong IF guarantees and does not necessarily hurt performance when used as a pre-processing step.}
\label{fig:overview}
\end{center}
\vskip -0.25in
\end{figure*}

\item \textbf{We empirically verify these results on real and synthetic datasets}. 
In Section \ref{sec:experiments}, we demonstrate our findings on synthetic data and the MovieLens 1M dataset. 
We visualize the effect of SVT pre-processing on IF.
Figure \ref{fig:lipschitz-diff}, for example, illustrates how the ratio $D(f(a), f(b)) / d(a, b)$ decreases under SVT pre-processing. 
Smaller values indicate a stronger IF guarantee. 
We also demonstrate the effect of SVT pre-processing on performance. 
\end{itemize}

To the best of our knowledge, this is the first work that establishes a theoretical link between IF and ME.

\section{Related Work}\label{sec:rw}
\textbf{Matrix estimation (ME)}. 
ME studies the problem of estimating the entries of a matrix from noisy observations of a subset of the entries \citep{candes2010power,recht2011simpler,keshavan2010matrix,negahban2012restricted,davenport20141,chatterjee,chen2015fast}. 
ME is a class of methods that can be applied to any data expressed in matrix form. 
Specifically, suppose there is a latent matrix, and one can only obtain noisy samples of a subset of its entries. 
The goal of ME is to estimate the values of every entry based on the noisy subsamples.

ME is used, for example, by recommender systems to estimate a user's interest in different types of content \citep{koren2009matrix,song2016blind,borgs2017thy}.
In fact, the winning solution of the Netflix Prize was built on ME methods \citep{koren2009bellkor}.
ME has also been used 
to study social networks \citep{anandkumar2013tensor,abbe2015community,hopkins2017efficient};  %
to impute and forecast a time series \citep{agarwal2018model,amjad2018robust}; 
to aggregate information in crowdsourcing \citep{shah2018reducing}; 
to improve robustness against adversarial attacks in deep learning \citep{yang2019me}; 
and more.\\

\noindent \textbf{Singular value thresholding (SVT)}.
There is an extensive literature on ME and the closely related areas of matrix completion and matrix factorization. 
While there are various approaches \citep{rennie2005fast}, 
spectral methods are among the most popular \citep{candes2010power,mazumder2010spectral,keshavan2010matrix,keshavan2010matrix2}

One such method is SVT \citep{cai2010singular}, which first factorizes the matrix of observations, then
reconstructs it using only the singular values that exceed a predetermined threshold. 
It is well-known that SVT is a shrinkage operator that provides a solution to a nuclear norm minimization problem.
\emph{Universal singular value thresholding} (USVT) builds on SVT by proposing an adaptive threshold that produces an estimator that is both consistent and approximately minimax \citep{chatterjee}. 
We review SVT and USVT in Sections \ref{sec:SVT} and \ref{sec:USVT}.\\

\noindent \textbf{Individual fairness (IF)}. 
 IF
 is the notion that similar individuals should receive similar treatment \citep{dwork2012fairness,barocas2018fairness}, 
 as formalized in \eqref{eq:if-gen}. 
As an example, suppose individuals A and B apply for job interviews at the same time with similar (observed) qualifications $a$ and $b$.
Then, IF requires that A and B receive interview requests at similar rates.
IF is distinct from notions of group fairness
(e.g., statistical parity in the outcomes across demographic groups), 
but there are conditions under which IF implies group fairness \citep{dwork2012fairness}.

Under IF, similarity is captured by the {choice} of distance metrics $D$ and $d$, and IF is enforced as a Lipschitz constraint based on the chosen metrics. 
How to define ``similarity'' between individuals and their outcomes (i.e., how to choose the distance metrics) has been the subject of significant debate  \citep{gajane2017formalizing,beutel2019fairness,ilvento2019metric,beutel2019fairness,gillen2018online,bechavod2020metric}.
In this work, 
we allow for any $D$. 
One of our IF results is given for $d$ as the $\ell^1$ norm
and the other for $d$ as the $\ell^q$ norm. \\

\noindent \textbf{Fairness and collaborative filtering}. 
In recommendation, collaborative filtering algorithms leverage similarities between users to infer user preferences, and ME can be viewed as one such algorithm. 
There is some work on the fairness of collaborative filtering, 
and these typically study group fairness \cite{kamishima2012enhancement,yao2017beyond,beutel2019fairness,foulds2020intersectional,pitoura2021fairness, shao2022faircf}. 
A small number of works examine notions of fairness related to individuals \cite{serbos2017fairness,biega2018equity,stratigi2020fair},
but they are distinct from our notion of IF as formulated by Dwork et al. \cite{dwork2012fairness}. 
To our knowledge, we provide the first theoretical analysis connecting IF to ME and  collaborative filtering, which can be found in Section \ref{sec:main_results}. \\ 

\noindent \textbf{Accuracy}. 
One common thread of interest in algorithmic fairness is the fairness-accuracy trade-off \cite{farnadi2018fairness,zhu2018fairness,liu2018personalizing,islam2020neural}. 
By establishing a connection between IF and USVT, 
we show in Section \ref{sec:USVT} that IF can be achieved without significant performance costs in ME applications, including collaborative filtering.

\section{Problem Statement}\label{sec:problem_statement}
\subsection{Setup}\label{sec:setup}

Consider a setting with $m$ individuals. 
Suppose there is an unknown ground truth matrix $A \in \bbR^{m \times n}$, 
where each row in $A$ corresponds to an individual such that the $i$-th row $\bA_i \in \bbR^{n}$ 
is an unknown $n$-dimensional feature vector that
describes individual $i \in [m]$. 
Without loss of generality, suppose that $A_{ij} \in [-1,1]$ for all $i \in [m]$ and $j \in [n]$.\footnote{
For any $A$ whose entries are finite such that $| A_{ij} | < \infty$ for all $i \in [m]$ and $j \in [n]$, one can always translate and rescale $A$ to be between $-1$ and $1$, then adjust the final result accordingly.}

Suppose that it is possible to observe a noisy subsample of $A$'s entries.
Formally, 
let $\Omega \subset [m] \times [n]$ denote the index set of observed entries 
and  $\cZ = [-1,1] \cup \{\emptyset\}$.
Let $Z \in \cZ^{m \times n}$ denote
the matrix of observations, 
where each entry of $Z$ is a random variable,
$\bbE Z_{ij} = A_{ij}$ if $(i,j) \in \Omega$, 
and
$Z_{ij} = \emptyset$, otherwise. 
As such, 
the $i$-th row $\bZ_i \in \cZ^{n}$ denotes the observed covariates for individual $i$. 
For the remainder of this work, let $\mathbf{B}_i$ denote the $i$-th row and $\mathbf{b}_i$ denote the $i$-th column of a matrix $B$.
\\

\noindent 
\textbf{Inference task}.
We consider the following inference task.  
Make a prediction $\by \in \cY$ for individual $i \in [m]$ 
using the observations (i.e., training data) $Z$. 
Let $\cF = \{ f : [m] \times \cZ^{m \times n} \rightarrow \cY \}$ denote the class of algorithms that perform this inference task. 
Note that the output of $f$ could be a deterministic value or a distribution over possible values.

\subsection{Individual Fairness}\label{sec:IF}

{Individual fairness} (IF) is the notion that \emph{similar individuals should receive similar treatments} \citep{dwork2012fairness}. 
IF is formulated as a $(D,d)$-Lipschitz constraint, as follows.

\begin{definition}[IF with respect to observed covariates]\label{def:IF_z}
Consider an observation matrix $Z \in \cZ^{m \times n}$. 
An algorithm $f \in \cF$ is $(D, d)$-\emph{individually fair on $Z$} if 
    \begin{align}
	D( f( i , Z ), f(j , Z ) ) \leq L \cdot d(\bZ_i , \bZ_j) \hspace{0.2in} \forall i,j \in [m], \label{eq:lipschitz_constraint_Z}
    \end{align}
 where $L \geq 0$ does not depend on $i$ or $j$, 
 $D$ is a metric on $\cY$, 
 and $d$ is a metric on $\cZ^n$. 
\end{definition}

\begin{definition}[IF with respect to latent covariates]\label{def:IF_a}
Consider an observation matrix $Z \in \cZ^{m \times n}$ and ground truth matrix $A \in [-1,1]^{m \times n}$.
An algorithm $f \in \cF$ is $(D, d)$-\emph{individually fair on $A$} if 
	\begin{align}
		D( f( i , Z), f(j , Z ) ) \leq L \cdot d(\bA_i , \bA_j)  \label{eq:lipschitz_constraint_a} \hspace{0.2in} \forall i,j \in [m],
	\end{align}
  where $L \geq 0$ does not depend on $i$ and $j$, 
  $D$ is a metric on $\cY$, 
  and $d$ is a metric on $[-1, 1]^n$. 
\end{definition}
\vspace{6pt}

\noindent 
\textbf{Problem statement}. 
We focus on a subclass of algorithms $\cF( \cH , \Pi ) = \{  f = h \circ \Pi : h \in \cH \} \subset \cF$, where $\cH \subset \{ h : [m] \times [-1,1]^{m \times n} \rightarrow \cY \}$
and $\Pi : \cZ^{m \times n} \rightarrow [-1,1]^{m \times n}$. 
Intuitively, $\Pi$ is a pre-processing algorithm that takes in the (sparse and noisy) data $Z$ and produces an estimate $\Pi(Z)$ of the unknown ground truth matrix $A$. 
The inference algorithm $h$ is then applied on top of $\Pi$ such that $f(i , Z) = h(i , \Pi(Z))$. 
In this work, 
we examine the IF of $f$ relative to $h$ when 
$\Pi$ is given by a ME method, 
i.e., \emph{how a ME pre-processing step affects the IF of an inference algorithm}. 

\subsection{Examples} \label{sec:examples}

The setup in Section \ref{sec:setup} can be applied to many problems in which the training data and algorithmic inputs are \emph{noisy}, \emph{sparse}, or both. 
Consider the following examples and the implications of IF. 

\begin{example}[Recommendation]\label{ex:recommendation}
	Consider a platform that provides personalized movie recommendations to its $m$ users based on sparse, noisy observations of their preferences. 
	Suppose that the movie preferences of each user $i \in [m]$ can be described by an unknown $n$-dimensional vector $\bA_i \in \bbR^n$. 
	For instance, $a_{ij} \in [-1, 1]$ could denote the ground-truth preference of user $i$ for movie $j \in [n]$. 
    Although $A = [\bA_1 ,  \hdots , \bA_m]^\top$ is unknown, 
	the platform receives occasional feedback from users in the form of ratings
	and can also observe the users' viewing behaviors. 
	Let these sparse, noisy observations be stored in $Z$, where
	$Z_{ij} = \emptyset$ implies that user $i$ has not rated movie $j$. 
 
    The goal of the platform is to estimate the users' movie preferences. 
	Note that $f \in \cF$ can leverage other information (e.g., ratings by other users, as done in collaborative filtering). 
	In this example, IF on $Z$ requires that users with similar viewing and rating behaviors receive similar recommendations. 
	IF on $A$ implies that users with similar latent (i.e., unknown) movie preferences receive similar recommendations. 
\end{example}

\begin{example}[Admissions]\label{ex:admissions}
	Consider an admissions setting in which there are $m$ applicants. 
	Suppose that, for the purposes of admissions, each applicant $i \in [m]$ is described by an unknown $n$-dimensional vector $\bA_i \in \bbR^n$.
	Suppose each individual $i$ submits an application $\bZ_i$, which contains sparse, noisy measurements of $\bA_i$. 
	For example, one's standardized test score in math is a noisy measurement of one's math abilities. 
	Data sparsity can occur when one applicant includes information that another does not (e.g., one may list ``debate club'' on their resume while another does not, but this sparsity does not necessarily imply that the latter is worse at public speaking). 
	As an output,  $f \in \cF$ could produce an admissions score $\by \in [0,1]$. 
	In this example, IF on $Z$ requires that applicants with similar applications receive similar admissions scores. 
	IF on $A$ implies that applicants whose true (but unknown) qualifications are similar receive similar admissions scores.
\end{example}
Although IF on $A$ is desirable, 
one generally requires IF on $Z$, 
i.e., that an algorithm ensures IF with respect to the information at its disposal. 
Consider Example \ref{ex:admissions}. 
Suppose that two applicants $i$ and $j$ have similar ground-truth features but the first $n/2$ values of $\bZ_{i}$ are $\emptyset$ while last $n/2$ values of $\bZ_{j}$ are $\emptyset$. 
In other words,
the types of qualifications that $i$ reports contains no overlap with the types of qualifications $j$ reports. 
Because $i$ and $j$ have similar ground-truth features, 
IF on $A$ would require that a school treat $i$ and $j$ similarly even though the schools are given vastly different information about the two applicants. 

\section{Main Results}\label{sec:main_results}
In this section, 
we show that pre-processing data with ME can improve IF with little to no performance cost under appropriate conditions. 
Before providing our main results, 
we begin in Section \ref{sec:SVT} by describing a ME method known as singular value thresholding (SVT). 
In Sections \ref{sec:SVT_IF_Z}-\ref{sec:SVT_IF_A}, we show that SVT pre-processing offers IF guarantees on both the observation matrix $Z$ and the ground truth matrix $A$. 
In Section \ref{sec:USVT}, 
we show that the class of SVT thresholds that guarantee IF align with the thresholds used by a well-known ME technique that has strong performance guarantees. 
This connection implies that SVT pre-processing can provide IF without imposing a high performance cost.

\subsection{Singular Value Thresholding}\label{sec:SVT}

Recall the inference task described in Section \ref{sec:setup}. 
In this section, 
we propose to use a popular ME method known as \textbf{singular value thresholding (SVT)} as the pre-processing step. 
That is, for algorithms in the class $\cF( \cH , \Pi ) = \{  f = h \circ \Pi : h \in \cH \} \subset \cF$, we propose that $\Pi$ denote SVT.

More precisely, 
$\svt(Z, \tau, \psi)$ takes in three values:
the observation matrix $Z\in \cZ^{m \times n}$,
a threshold $\tau \geq 0$, 
and an increasing function $\psi : \bbR_{\geq 0} \rightarrow \bbR_{\geq 0}$. 
SVT then proceeds in four steps:
\begin{enumerate}[topsep=9pt,itemsep=1pt]
    \item For any element in $Z$ that is $\emptyset$, 
    replace that value with $0$, 
    i.e., if $Z_{ij} = \emptyset$, 
    re-assign it to $Z_{ij} = 0$. 

    \item Perform the singular value decomposition (SVD):
    \begin{align*}
        Z = \sum_{\ell=1}^{\min (m,n)} \sigma_\ell \bu_\ell \bv_\ell^T ,
    \end{align*}
    where $\sigma_\ell \geq 0$ is the $\ell$-th singular value,
    $\bu_\ell \in \bbR^{m \times 1}$ is the $\ell$-th left singular vector,
    and
	$\bv_\ell \in \bbR^{n \times 1}$ is the $\ell$-th right singular vector.
 
    \item 
    For any index $\ell$ such that $\sigma_\ell > \tau$, 
    add $\ell$ to the set $S(\tau)$ such that $S(\tau) = \{ \ell : \sigma_\ell > \tau \}$.
    
    \item Finally, construct an estimate of $A$:
    \begin{align*}
        \hat{A} = \min \bigg( 1 , \max \bigg( -1 ,  \sum_{\ell \in S(\tau)} \psi( \sigma_\ell ) \bu_\ell \bv_\ell^T \bigg) \bigg) .
    \end{align*}
\end{enumerate}
Intuitively, SVT detects and removes components of the observation matrix $Z$ that correspond to noise while preserving the remaining components $S(\tau)$.  
The threshold $\tau$ determines the boundary between signal and noise, 
where a higher value for $\tau$ means that fewer components are kept.

\subsection{IF With Respect to Observed Covariates} \label{sec:SVT_IF_Z}

In the previous section, we proposed to pre-process $Z$ using  SVT  before applying an inference algorithm $h$ on top of it. 
In this section, we show that using SVT for pre-processing guarantees IF on $Z$.
For the remainder of this section, we fix the $Z$ of interest.

Consider a specific threshold $\tau$ and function $\psi$. 
Recall that $\sigma_\ell$, $\bu_\ell$, and $\bv_\ell$ 
are the $\ell$-th singular value, left singular vector, and right singular vector of $Z$,
respectively.
Recall further that $S(\tau) := \{ \ell : \sigma_\ell > \tau \}$.
Let  
\begin{align*}
    K_2 = \norm{\sum_{\ell \in S(\tau)}\frac{\psi(\sigma_\ell)\bu_\ell \bv_\ell^T}{\sigma_\ell^2}}_{\infty} \sqrt{n} \max_k \norm{\bz_k}_1. 
\end{align*}

\begin{theorem} \label{thm:IF_z}
    Suppose that $h$ is $(\cD, \ell^2)$-individually fair with constant $K_1$, i.e., 
    $$D(h(i, B) , h(j , B)) \leq K_1 || \mathbf{B}_i - \mathbf{B}_j ||_{2},$$ 
    for all $i , j \in [m]$ and $B \in [-1, 1]^{m \times n}$.
    Then, for $f = h \circ \svt(Z , \tau, \psi)$,
	\begin{equation}
		D(f ( i , Z) , f( j , Z) ) 
		\leq K_1 K_2 \norm{\bZ_i - \bZ_j}_1 ,
	\end{equation}
	for all $i ,j \in [m]$, 
    i.e., $f$ is $(D, \ell^1)$-individually fair on $Z$ with constant $K_1 K_2$.
\end{theorem}
Theorem \ref{thm:IF_z} states that when $h$ is IF with Lipschitz constant $K_1$, applying SVT pre-processing preserves IF with constant $K_1 K_2$ with respect to the observed covariates. In order for $h$ with SVT pre-processing to have stronger IF than $h$ alone, we need $K_2 \ll 1$, as we examine next. 

\begin{corollary}\label{cor:k1}
    Suppose $\psi(x) = \beta x$ and $Z$ satisfies the strong incoherence condition\footnote{Strong incoherence is a standard assumption in the ME literature \cite{keshavan2010matrix, negahban2012restricted, chen2015incoherence}. 
    It requires that the singular vectors of a matrix are not sparse, which can make it difficult to estimate the underlying latent matrix when given limited samples.}  with parameter $\mu_1$,
    i.e., 
    \begin{align*}
        \norm{\sum_{\ell \in S(\tau)} \bu_\ell \bv_\ell^T}_\infty \leq \sqrt{\frac{\mu_1 r}{mn}},
    \end{align*}
    where $r = |S(\tau)|$ denotes the rank of $Z$. 
    Then for any threshold $\tau$, $K_2 \leq {\beta \sqrt{rm}} / {\tau}$.
\end{corollary}
Corollary \ref{cor:k1} characterizes common conditions under which $K_2$ scales as $O(\sqrt{r m} / \tau)$. 
Specifically, 
suppose that $\tau \geq \sqrt{2n\beta }$.
Then, $K_2 = O (\sqrt{{rm}/{n}} )$. 
This indicates that combining $h$ with SVT pre-processing would \emph{improve} the IF of $h$ as long as $n = \omega(rm)$. 
In  other words, 
as long as there is enough data $n$ per individual relative to the number of individuals $m$ and the rank $r$,
then $K_2 \rightarrow 0$ as $n \rightarrow \infty$.\footnote{The rank $r$ indicates the ``complexity'' of the ground-truth matrix $A$. Although it is computed using $Z$, it reflects the amount of ``signal'' in $Z$, which generally depends on $A$.}
We discuss the implications of this result further in Section \ref{sec:discussion}. 

\subsection{IF With Respect to Latent Covariates} \label{sec:SVT_IF_A}

In the previous section, 
we showed that SVT pre-processing can improve IF on $Z$. 
In this section, 
we show that SVT pre-processing can also ensure IF on $A$ as long as its estimates $\hat{A}$ are close to the ground-truth values. 

\begin{theorem} \label{thm:IF_a}
    Let $d$ denote the $\ell^q$ norm. 
	Suppose that $h$ is $(\cD, d)$-individually fair with constant $K_1$, i.e., 
    $$D(h(i, B ) , h(j , B )) \leq K_1 || \mathbf{B}_i - \mathbf{B}_j ||_{q},$$ 
    for all $i , j \in [m]$ and $B \in [-1, 1]^{m \times n}$.
     Then, for $f = h \circ \svt(Z , \tau, \psi)$,
	\begin{align}
		D(f (i, Z & ) ,  f( j , Z) )  
		\leq  K_1 \norm{\bA_i - \bA_j}_q 
  + 2 K_1 || \hat{A} - A ||_{q,\infty} , \label{eq:IF_a}
	\end{align} 
    for all $i ,j \in [m]$.
\end{theorem}

Theorem \ref{thm:IF_a} states when $h$ is IF with Lipschitz constant $K_1$, then $f$ is approximately IF on $A$ and approaches exact IF as $\hat{A} \rightarrow A$. 
Note that Theorem \ref{thm:IF_a} holds for any $\Pi$. 
This result implies that SVT pre-processing preserves the individual fairness guarantee of $h$ on $A$ as the estimation error of SVT approaches $0$. 
We show in the next section (Proposition \ref{prop:USVT_MSE}) that, under an appropriate choice of threshold, 
the estimation error of SVT indeed goes to $0$ (specifically, that $||\hat{A} - A||_{2,\infty} \rightarrow 0$) as $m , n \rightarrow \infty$. 
Together, 
these two results imply that adding SVT pre-processing to $h$ ensures IF on $A$ under the same conditions that guarantee that SVT (or, more generally, ME) is accurate.\footnote{
	Note that the condition in both theorems
	that $D(h(i, B) , h(j , B)) \leq K_1 || \mathbf{B}_i - \mathbf{B}_j ||_{q}$ for all $i , j \in [m]$ and $B \in [-1, 1]^{m \times n}$ is not strong.
	In fact, if it is not met, then there is no method $\Pi$ such that $f$ is IF. 
}
\vspace{6pt}
\begin{remark}
    Theorem \ref{thm:IF_a} shows that it is possible to achieve approximate IF on $A$, 
    and the tightness of this guarantee depends on the accuracy of $\Pi$.
    Even though IF on $A$ may be desirable, 
    IF on $Z$ is important because both individuals and algorithm designers generally cannot make claims based on the unknown ground-truth matrix $A$; 
    they must point to the evidence (i.e., observations) $Z$.
\end{remark}

\subsection{Performance Under Individual Fairness}\label{sec:USVT}

Recall from Theorem \ref{thm:IF_z} that, 
as long as the threshold $\tau$ is sufficiently large, SVT pre-processing guarantees IF on $Z$.
However, it is unclear if the threshold chosen for IF is good for prediction performance. 
We now show that an adaptive threshold that is known to provide high accuracy coincides with thresholds that guarantee IF on $Z$. 
Because this adaptive threshold guarantees that $\hat{A} \rightarrow A$,
it also guarantees IF on $A$, as per Theorem \ref{thm:IF_a}.
As a result, 
SVT pre-processing under the appropriate threshold ensures IF on both $Z$ and $A$ at little to no performance cost.

Consider a well-known ME method known as \textbf{universal singular value thresholding (USVT)}. 
USVT refines SVT by proposing a universal formula for the threshold $\tau$, 
thereby removing the need to tune $\tau$ by hand.
Under mild assumptions on $A$ and $\Omega$,
USVT has strong performance guarantees. %
In order to study performance, 
let the mean-squared error (MSE) of ME be defined as
\begin{align}
	\mse(\htA) \defeq \frac{1}{mn}  \sum_{i=1}^m \sum_{j=1}^n \bbE \left[(\htA_{ij} - A_{ij})^2\right] . \label{eq:mse}
\end{align}
Let $\norm{M}_*$ denote the nuclear norm of  matrix $M$. 
We begin with a performance guarantee on USVT. 

\begin{proposition}[Modified from Theorem 1.1. in \citet{chatterjee}]\label{prop:USVT_MSE}
	Suppose that the entries of $A$ are independent random variables. 
	Suppose each entry of $A$ is independently observed with probability $p \in [0,1]$. 
	Let $\hat{p}$ be the proportion of observed values, $\psi( x) = x / \hat{p}$, $\epsilon \in (0,1]$, and $w = (2 + \eta)^2$ for $\eta \in (0,1)$. 
	Let $\rho_1 = \max(m,n)$ and $\rho_2 = \min(m,n)$.
	Then, if $p \geq \rho_1^{\epsilon - 1}$ and $\tau = \sqrt{w \rho_1 \hat{p}}$, 
	\begin{align*}
		\mse \left( \svt(Z, \tau, \psi) \right) \leq &C(\eta) \min \left(
		\frac{\norm{A}_*}{\rho_2 \sqrt{\rho_1 p}} , 
		\frac{\norm{A}_*^2}{ \rho_1 \rho_2 } , 
		1
		\right) + C(\epsilon, \eta) \exp( - c (\eta) \rho_1 p ) ,
	\end{align*} 
	where $C(\eta) , c(\eta) > 0$ depend only on $\eta$
	and $C(\epsilon, \eta)$ depends only on $\eta$ and $\epsilon$.\footnote{\label{fn:MSE_USVT}
		This upper bound can be improved when the additional condition that
		$\var(Z_{ij}) \leq \sigma^2$ for all $i,j$ and $\sigma \leq 1$ holds. 
		Then, if  $\tau \geq \sqrt{w n \hat{q}}$, 
		where $\hat{q} = \hat{p} \sigma^2 + \hat{p}(1 - \hat{p}) (1 - \sigma^2)$, 
		$q \geq n^{\epsilon - 1}$, 
		and $q = p \sigma^2 + p(1 - p) (1 - \sigma^2)$:
			\begin{align*}
				\mse( \htA ) &\leq C(\eta) \min \left(
				\frac{\norm{A}_* \sqrt{q} }{m p \sqrt{n}} , 
				\frac{\norm{A}_*^2}{m n } , 
				1
				\right) + C(\epsilon, \eta) \exp( - c (\eta) n q ) .
			\end{align*}
		}
\end{proposition}

Proposition \ref{prop:USVT_MSE} states that when $\tau = \sqrt{w \rho_1 \hat{p}}$ and $p$ is large enough,
the MSE of SVT decays at a rate of $o((mn)^{-1})$. 
As an immediate extension,  Proposition \ref{prop:USVT_MSE} tells us that if the loss of $h$ when given perfect information $A$ is small, 
    then the loss of $f = h \circ 
    \textsc{SVT}(Z, \sqrt{\omega \rho_1 \hat{p}}, \psi)$ is also small as $n , m \rightarrow \infty$ because the estimate $\hat{A}$ produced by USVT is close to $A$. 
    
\begin{remark}
    \citet{chatterjee} also show that the MSE of USVT is within a constant multiplicative factor and an exponentially small, additive term of the MSE of the minimax estimator, which implies that one cannot do much better than the USVT (cf. Theorem 1.2 in \citet{chatterjee}). 

\end{remark}
As such, 
SVT is consistent and approximately minimax under the appropriate choice of threshold. 
Next, we connect this finding to our earlier results on IF. \\

\noindent \textbf{Performance under IF on $Z$.}
Suppose that $n > m$. 
Then, $\rho_1 = n$ and 
Theorem \ref{prop:USVT_MSE} indicates that SVT pre-processing with the threshold $\tau = \sqrt{w \hat{p} n}$ has good performance.
Under Corollary \ref{cor:k1}, 
such a threshold also ensures that $f$ with SVT pre-processing is \emph{more} individually fair on $Z$ than $f$ without SVT pre-processing for large enough $n$ such that $n = \omega(r m)$.
Therefore, 
there is no trade-off between performance and IF under SVT pre-processing when $n$ grows at the rate $\omega(rm)$. \\

\noindent \textbf{Performance under IF on $\mathbf{A}$}.
Recalling Theorem \ref{thm:IF_a}, 
ME is approximately individually fair on $A$
and fully individually fair on $A$ when $|| \Pi(Z) - A ||_{q , \infty} = 0$. 
Therefore, the relationship between IF on $A$ and performance under ME is straightforward: 
the lower the estimation error $\norm{ \Pi(Z) - A }_{q , \infty}$,
the more individually fair $f$ is on $A$.

\section{Discussion}\label{sec:discussion}
In this section, 
we interpret the results and discuss the conditions under which SVT pre-processing guarantees IF and good performance simultaneously. \\

\noindent \textbf{Combining the results}.
Under Proposition \ref{prop:USVT_MSE}, 
SVT yields good performance guarantees as $n \rightarrow \infty$ when $\tau = \sqrt{w \hat{p} n}$ and $n \geq m$. 
Under Corollary \ref{cor:k1},
this same $\tau$ guarantees IF on $Z$ with Lipschitz constant $K_1 K_2$, 
where $K_1$ is the Lipschitz constant for $h$ without SVT pre-processing and $K_2 = O(\sqrt{r m  / (n\hat{p})})$.
SVT pre-processing can improve IF on $Z$ \emph{without} sacrificing performance when $K_2 \ll 1$, So, when is $K_2 \ll 1$, and why is $K_2$ sometimes greater than $1$?
To answer this question, 
we examine two data regimes: 
(i) \emph{when $n = o(rm / \hat{p} )$} and (ii) \emph{when $n = \omega(rm / \hat{p} )$}. \\

\noindent \textbf{First data regime}.
In the first data regime, 
Corollary \ref{cor:k1} tells us that $K_2 > 1$, 
which implies that SVT pre-processing does \emph{not} necessarily improve IF. 
This phenomenon occurs because, when there is not much information by which to distinguish between individuals (i.e., $n$, the number of observed features per individual, is small), 
SVT pre-processing produces an $\hat{A}$ that is smoothed across rows. 
That is, it causes $f$ to treat individuals similarly \emph{on the whole}. 

This can, at times, work against IF, which requires that \emph{similar} individuals be treated similarly, but not that \emph{the population} be treated similarly. 
To see why the latter can work against IF, 
consider $g_1(x) = x$ and $g_2(x) = \text{round}(x)$ for $x \in [0, 1]$. 
Under $g_2$, individuals can only receive outcomes $0$ or $1$, so the algorithm treats individuals similarly \emph{on the whole}. By this, we mean that individuals fall into one of two buckets, so the treatment is relatively homogeneous.

On the other hand, under $g_1$, individuals receive one of infinitely many outcomes in the range $[0, 1]$.
Which of the two is individually fair?
Although $g_2$ treats individuals similarly on the whole,
$g_1$ is IF since $d(g_1(x), g_1(x')) = d(x, x')$
while $g_2$ is \emph{not} because $g_2(0.5 - \delta) = 0$ while $g_2( 0.5 + \delta) = 1$ for arbitrarily small $\delta > 0$.
A similar logic can be used to show that SVT pre-processing does not always improve IF in this first data regime.\footnote{
    Although SVT pre-processing does not necessarily improve IF in the first data regime,
    some might argue that the ``smoothing'' that SVT does can prevent $f$ from unnecessarily differentiating between individuals. 
    For example, suppose that $f$ determines how much COVID-19 relief each household gets. 
    Suppose that, due to the short turnaround time, $n = o(r^2 m)$, e.g., the government has little information on how each household has been affected by COVID-19.
    One might argue that, in such situations, 
    the government cannot reliably distinguish between households and should send the same amount of monetary relief to all households rather than tailor the amounts based on limited data.
	The reasoning goes: in this data regime, 
    it is easy to overfit and use spurious information to distinguish between individuals. 
    In this way, 
    one may debate the importance of IF in the first data regime. 
} \bigskip 

\noindent \textbf{Second data regime}.
In the second data regime, 
Corollary \ref{cor:k1} tells us that $K_2 < 1$, 
which implies that SVT pre-processing \emph{improves} IF. 
Intuitively, 
when $n = \omega(rm / \hat{p})$, 
the expected number of observed features per individual grows faster than the number of individuals and rank of the ground truth matrix. 
In this case, 
SVT smooths the data in a different way. 
It produces an $\hat{A}$ that is smoothed across columns. 
It therefore removes noise from individual (row) vectors $\bZ_i$ but leaves enough signal in $\bZ_i$ to differentiate individual $i$ from other individuals, 
thereby avoiding the phenomenon that can occur in the first data regime (that individuals are treated similarly on the whole).
The fact that the observational data is smoothed but individuals remain differentiable allows SVT to improve IF in this data regime. \bigskip

\noindent \textbf{Putting it together}.
SVT pre-processing smooths the data before sending it to $h$, 
and this smoothing operation affects IF differently in different data regimes.
We show, however, that under an appropriately chosen threshold, 
IF on $Z$, IF on $A$, and good performance are simultaneously guaranteed as $n \rightarrow \infty$. 
More precisely,
\emph{when $\tau = \sqrt{w \hat{p} n}$, 
$n = \omega(r m / \hat{p})$, 
and $n$ is sufficiently large, 
SVT pre-processing not only strengthens IF on $Z$, but it also guarantees IF on $A$ and good prediction performance.}

\section{Experiments}\label{sec:experiments}
We provide several experiments that test the effect of SVT pre-processing on IF and performance.
In each experiment,
the inference task is to estimate the unknown $n$-dimensional feature vector $\bA_i$ for each individual $i \in [m]$ using the observations $Z$. 
The results show that SVT pre-processing improves IF, 
both in simulation and in the MovieLens1M dataset.
We also examine the performance of an inference algorithm with and without SVT pre-processing. 
As expected, 
we find that adding SVT pre-processing increases the MSE but only by a small amount; 
by Theorem \ref{prop:USVT_MSE}, we would expect this amount to decay to $0$ as the amount of data grows. 

Below, we divide our discussion into three parts. 
In the first two parts, we describe our experimental setups for the \emph{synthetic data} and on the \emph{MovieLens 1M dataset}.
In the third part, we discuss the results. 
Additional results and implementation details can be found in the Appendix.

\subsection{Setup for Experiment \#1: Synthetic Data} \label{sec:synth-setup}

In Experiment \#1, 
we test $h$ with and without SVT pre-processing on synthetic data, as follows. \\  

\noindent \textbf{Generating the ground truth matrix $A$}.  
Consider $m = 200$ individuals. 
We sample $m$ feature vectors of length $n = 800$, each corresponding to an individual, to form the ground truth matrix $A \in [-1, 1]^{m \times n}$. The feature vectors (i.e., the rows of $A$) are sampled from $c = 10$ clusters, where each cluster is a multivariate normal distribution. The mean of each cluster is a vector of length $n$ drawn uniformly at random from $(-1, 1)$, and the covariance of each cluster is an $n \times n$ diagonal matrix with whose diagonal values are sampled uniformly at random from $(0, 0.1)$.
The feature vectors are then clipped so that all values fall within $[-1, 1]$. \\ 

\noindent \textbf{Generating the observation matrix $Z$}.
Recall that $\Omega$ denotes the set of observed entries.
We generate $Z$ as follows:
\begin{align*}
    Z_{ij} = \begin{cases}
        \text{clip}(A_{ij} + \eta_{ij}, [-1, 1]) , & \text{if } (i, j) \in \Omega ,
        \\
        \emptyset , & \text{otherwise} ,
    \end{cases}
\end{align*}
where $\eta_{ij} \sim \mathcal{N}(0, 0.1)$.
In this section, $(i, j) \in \Omega \subset [m] \times [n]$ with probability $p$. 
This is aligned with the conditions in Proposition \ref{prop:USVT_MSE}. 
In the Appendix, 
we provide results under a different choice of $\Omega$ (specifically, when the probability of observing an individual $i$'s $j$-th feature depends on the cluster to which $i$ belongs). \\

\noindent \textbf{Inference algorithm}. 
Recall that the inference task is to predict the feature vector $\mathbf{A}_i$ for individual $i$ given data $B$ (where $B$ may or may not have undergone SVT pre-processing).
In the synthetic data setting, we let the  algorithm $h : [m] \times [-1, 1]^{m \times n} \rightarrow [-1, 1]^n$ be given as follows.

Let $h' : [m] \times [n] \rightarrow [0, 1]$ denote a deep neural net (DNN) trained on data $B$. 
Let the DNN be composed of three fully connected layers of size 300, 100, and 1 with ReLU activation after the hidden layers and sigmoid after the output layer. 
Lastly, let 
\begin{align*}
    h(i, B) = 2 [ h'(i, 1), h'(i, 2), \hdots, h'(i, n)]^\top - 1.
\end{align*} 

\noindent \textbf{Pre-processing}. 
We compare the IF and performance of $h$ with and without SVT pre-processing. 
When there is no pre-processing step, 
the data $B$ on which $h'$ is trained is $Z$ (missing entries are replaced with zeros). 
When SVT pre-processing is used, 
the data $B$ on which $h'$ is trained is $\svt(Z, \tau, \psi)$, 
where 
$\hat{p} = |\Omega|/(mn)$,
$\psi(x) = x/\hat{p}$,  
$\hat{q} =  0.01^2 \hat{p} + \hat{p}(1 - \hat{p}) (1 - 0.01^2)$, and
$\tau = \sqrt{2.01 n \hat{q}}$.
This form of SVT is consistent with USVT (see Proposition \ref{prop:USVT_MSE} and Footnote \ref{fn:MSE_USVT}).

\subsection{Setup for Experiment \#2: MovieLens 1M Dataset}

In Experiment \#2, 
we test $h$ with and without SVT pre-processing on a popular, real-world dataset known as the MovieLens 1M Dataset.

\textbf{Dataset}.
The MovieLens 1M dataset \cite{harper2015movielens} contains movie ratings data for 6040 users and 3952 movies. 
In the context of this work, this ratings data can be placed in the $m \times n$ matrix $Z$, 
where $m = 6040$ and $n = 3952$.
Each entry $Z_{ij}$ contains user $i$'s rating of movie $j$ if $(i,j)$ is observed, 
and $Z_{ij} = \emptyset$ if user $i$ has not rated movie $j$. 
The ratings are normalized to be between $0$ and $1$. 

As a real-world dataset, 
there is no ground-truth matrix $A$. 
As such, we cannot evaluate performance relative to $A$---our MovieLens discussion instead focuses on IF. 

\textbf{Inference algorithm}. 
Recall that the inference task is to predict the feature vector $\mathbf{A}_i$ for individual $i$ given data $B$ (where $B$ may or may not have undergone SVT pre-processing).
In the MovieLens setting, 
we let the inference algorithm $h : [m] \times [-1, 1]^{m \times n} \rightarrow [-1, 1]^n$ be the $K$-nearest neighbors ($K$-NN) algorithm \cite{sarwar2001item}.\footnote{We use $K$-NN in order to investigate the effect of SVT pre-processing on another common class of algorithms. In particular, $K$-NN smooths data in a way that already encourages IF, 
which makes it particularly meaningful if SVT pre-processing is able to \emph{further} improve IF.}

$K$-NN produces an estimate $\bY_i$ by taking the weighted average of the $K$ users most similar to user $i$. 
In this work, we let $K = 10$ and the similarity between users $i$ and $j$ be measured using adjusted cosine similarity:
\begin{equation*}
    \text{sim}(i,j) = \frac{\sum_{k \in [n]} (B_{ik} - \bar{B}_k)(B_{jk} - \overline{B}_k)}{\sqrt{\sum_{k \in [n]} (B_{ik} - \bar{B}_k)^2 \sum_{k' \in [n]} (B_{jk'} - \bar{B}_{k'})^2}},
\end{equation*}
where $\bar{B}_k$ represents the average of the $k$-th item's ratings.

\textbf{Pre-processing}. 
We compare the IF of $h$ with and without SVT pre-processing. 
When there is no pre-processing step, 
the data $B$ used by $K$-NN is $Z$ (missing entries are replaced with zeros). 
When SVT pre-processing is used, 
the data $B$ used by $K$-NN is $\svt(Z, \tau, \psi)$, 
where 
$\hat{p} = |\Omega|/(mn)$,
$\psi(x) = x/\hat{p}$,  
$\tau = \sqrt{2.01 n \hat{p}}$, 
as consistent with USVT (see Proposition \ref{prop:USVT_MSE}).

\subsection{Results}

\begin{table}[t]
\caption{Results on IF and performance in Experiment \#1.}
\label{synth-random}
\vskip 0.15in
\begin{center}
\begin{small}
\begin{sc}
\begin{tabular}{lccccr}
\toprule
{} & {$\hat{p} = 0.05$} & {$\hat{p} = 0.1$} & {$\hat{p}=0.2$} & {$\hat{p}=0.4$}\\
\midrule
    {MSE($h$)}  & {$0.33 \pm 0.003$} & {$0.21 \pm 0.002$} & {$0.10 \pm 0.003$} & {$0.07 \pm 0.001$} \\
\midrule
    {MSE($f$)}  & {$0.34 \pm 0.004$} & {$0.21 \pm 0.001$} & {$0.11 \pm 0.001$} & {$0.08 \pm 0.001$} \\
\midrule
    {$\IF^h_1(Z)$}  & {$0.23 \pm 0.005$} & {$0.18 \pm 0.003$} & {$0.13 \pm 0.001$} & {$0.07 \pm 0.001$} \\
\midrule
    {$\IF^f_1(Z)$}  & {$0.02 \pm 0.001$} & {$0.02 \pm 0.001$} & {$0.03 \pm 0.001$} & {$0.03 \pm 0.001$} \\
\midrule 
    {\centering $K_2$}  & {$0.06 \pm 0.007$} & {$0.12 \pm 0.003$} & {$0.25 \pm 0.009$} & {$0.49 \pm 0.012$} \\
\midrule
    {$\IF^h_2(A)$}  & {$0.45 \pm 0.011$} & {$0.63 \pm 0.012$} & {$0.81 \pm 0.005$} & {$0.84 \pm 0.013$} \\
\midrule
    {$\IF^f_2(A)$}  & {$0.49 \pm 0.015$} & {$0.65 \pm 0.006$} & {$0.82 \pm 0.003$} & {$0.84 \pm 0.005$} \\
\bottomrule
\end{tabular}
\end{sc}
\end{small}
\end{center}
\vskip -0.2in
\end{table}

\textbf{Metrics}.
For a function
$g: [m] \times \cZ^{m \times n} \rightarrow \bbR$, let
\begin{align*}
	\mse(g) \defeq \frac{1}{mn - |\Omega|}  \sum_{(i, j) \notin \Omega} 
        (g_j(i, Z) - \mathbf{A}_{ij} )^2 , 
\end{align*}
where $g_j(i, Z)$ is the $j$-th element of the vector $g(i, Z)$.
For a matrix $X \in \bbR^{m \times n}$, 
let
\begin{align*}
    \IF^g_q(X) &\coloneqq \frac{1}{m^2} \sum_{i,j \in [m]} \norm{g(i,Z) - g(j,Z)}_2 / \norm{\mathbf{X}_i - \mathbf{X}_j}_q , 
\end{align*}
$\IF^f_1(Z)$ and $\IF^h_1(Z)$ measure IF on $Z$ with and without SVT pre-processing, respectively.
$\IF^f_2(A)$ and $\IF^h_2(A)$ measure IF on $A$ with and without SVT pre-processing, respectively.\footnote{We use the $\ell^1$ norm for $Z$ as per our result in Theorem \ref{thm:IF_z} and the $\ell^2$ norm for $A$ due to the connection between Theorem \ref{thm:IF_a} and Proposition \ref{prop:USVT_MSE}.}
A smaller ratio indicates a stronger IF guarantee. \\ 

\begin{figure}[t]
\begin{center}
\centerline{\includegraphics[width=0.9\columnwidth]{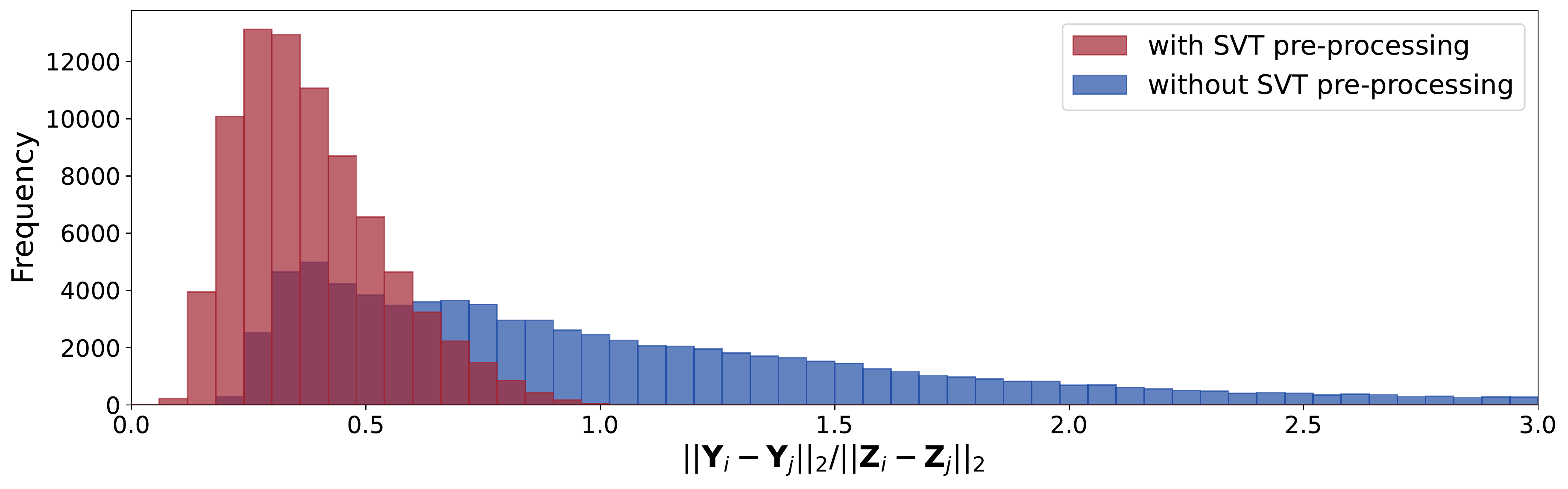}}
\vskip -0.1in
\caption{
Frequencies of  $||\bY_i - \bY_j||_2/\norm{\bZ_i - \bZ_j}_1$ across randomly selected pairs $(i, j)$ in Experiment \#2.
$Y$ denotes the estimate produced by $K$-NN on the MovieLens 1M dataset with (red) and without (blue) SVT pre-processing.}
\label{fig:movielens-knn}
\end{center}
\vskip -0.25in
\end{figure}

\noindent \textbf{Results}. 
Table \ref{synth-random} summarizes the results for Experiment \#1. 
The values are averaged over 10 simulations, 
and the error bars give $+/-$ two standard deviations.
Figures \ref{fig:lipschitz-diff}-\ref{fig:movielens-knn} visualize the effect of SVT pre-processing on IF on $Z$ for Experiments \#1 and \#2. 
We discuss our findings below. \\ 

\noindent \textbf{Effect of SVT pre-processing on IF on $Z$}.
Table \ref{synth-random} verifies that SVT pre-processing improves IF on $Z$ in Experiment \#1. 
In particular, $\IF^f_1(Z)$ is much smaller than $\IF^h_1(Z)$. 
Figures \ref{fig:lipschitz-diff} and \ref{fig:movielens-knn} visualize this effect for  Experiments \#1 and \#2, 
showing that SVT pre-processing causes the difference in two individuals' outcomes relative to the difference in their features to be smaller than without SVT pre-processing. \\ 

\noindent \textbf{Effect of SVT pre-processing on IF on $A$}.
IF on $A$ is comparable though slightly weaker with SVT pre-processing than without it. 
In particular, 
$\IF^f_2(A)$ is slightly larger than $\IF^h_2(A)$ in Table \ref{synth-random}.
This is in line with Theorem \ref{thm:IF_a}, which tells us that adding a pre-processing step $\Pi$ may weaken IF on $A$ if the estimation error of $\Pi$ is non-zero. 
Since SVT cannot estimate $A$ perfectly, it yields some estimation error and, as a result,  slightly weakens $f$'s IF on $A$. 

As illustrated in Table \ref{synth-random}, this effect is small. 
Moreover, the gap between $\IF^f_2(A)$ and $\IF^h_2(A)$ gets smaller as $\hat{p}$ increases. 
This is consistent with our results because the estimation error of SVT decreases as $\hat{p}$ increases (see Proposition \ref{prop:USVT_MSE}), 
which means that the IF on $A$ guarantee improves as $\hat{p}$ increases (see Theorem \ref{thm:IF_a}). \\

\noindent \textbf{Effect of SVT pre-processing on performance}.
The rows in Table \ref{synth-random} corresponding to $\mse(h)$ and $\mse(f)$ measure the error of the DNN without and with SVT pre-processing, respectively, in Experiment \#1. 
As expected from Section \ref{sec:USVT}, they show that SVT pre-processing has a minimal effect on prediction performance, 
i.e., that there is little to no fairness-performance trade-off.

\section{Conclusion}
In this work, we propose using a well-known matrix estimation (ME) method known as singular value thresholding (SVT) to pre-process sparse, noisy data before applying an inference algorithm (e.g., a neural network). 
We show that pre-processing data using SVT before applying an inference algorithm comes with strong individual fairness (IF) guarantees.
Specifically, we derive conditions under which SVT pre-processing \emph{improves} IF. 
We then show that, under these same conditions, 
SVT pre-processing has strong performance guarantees. 
Together, 
these results imply that, under the appropriate conditions, 
SVT pre-processing provides a way to improve IF without imposing a performance cost. 
We verify our results on synthetic data and the MovieLens 1M dataset. 

\section*{Acknowledgements}

We thank our reviewers for their time and helpful comments. We also thank Michael Zhang for providing feedback on earlier versions of this work. This work was supported in parts by the MIT-IBM project on ``Representation
Learning as a Tool for Causal Discovery" and the NSF TRIPODS Phase II grant towards Foundations of Data Science Institute.

\bibliographystyle{customref}
\bibliography{ref}

\newpage 
\appendix
\section{Appendix}
\subsection{Proof of Theorem \ref{thm:IF_z}}

In order to prove Theorem \ref{thm:IF_z}, we first prove the following two lemmas. 

\begin{lemma}\label{lemma:l2}
Suppose $T \in \mathbb{R}^{m \times n}$ and $\bx \in \mathbb{R}^{n \times 1}.$ Then
\[ \norm{T \bx}_2 \leq \norm{T}_\infty \norm{\bx}_1 \sqrt{m}. \]
\end{lemma}
\begin{proof}
\begin{align*}
    \norm{T \bx}_2 &= \left(\sum_{i \in [m]} \left( \sum_{j \in [n]} T_{ij} x_j \right)^2\right)^{1/2} 
    \leq \left(\sum_{i \in [m]} \left( \sum_{j \in [n]} |T_{ij}| |x_j| \right)^2\right)^{1/2} 
    \leq \left(\sum_{i \in [m]} \left( \sum_{j \in [n]} \norm{T}_\infty |x_j| \right)^2\right)^{1/2} \\
    &= \norm{T}_\infty \left(\sum_{i \in [m]} \left( \sum_{j \in [n]} |x_j| \right)^2\right)^{1/2} 
    = \norm{T}_\infty \sqrt{m} \left( \sum_{j \in [n]} |x_j| \right) 
    = \norm{T}_\infty \norm{\bx}_1 \sqrt{m}.
\end{align*}
\end{proof}

\begin{lemma}\label{lemma:l1}
Suppose $T \in \mathbb{R}^{m \times n}$ and $\bx \in \mathbb{R}^{n \times 1}$. Then 
\[ \norm{T \bx}_1 \leq \norm{\bx}_1 \max_j \norm{\bt_j}_1. \]
\end{lemma}
\begin{proof}
Recall $\bT_i$ denotes the $i$-th row of $T$ and $\bt_i$ denotes the $i$-th column of $T$.
\begin{align*}
    \norm{T \bx}_1 &= \sum_{i \in [m]} |T_i^\top \bx|
    = \sum_{i \in [m]} \left| \sum_{j \in [n]} T_{ij} x_j \right| 
    = \sum_{i \in [m]} \sum_{j \in [n]} |T_{ij} x_j|
    \leq \sum_{i \in [m]} \sum_{j \in [n]} |T_{ij}| |x_j| \\
    &= \sum_{j \in [n]} |x_j| \sum_{i \in [m]} |T_{ij}| 
    = \norm{\bx}_1 \max_j \left(\sum_{i \in [m]} |T_{ij}| \right)
    = \norm{\bx}_1 \max_j \norm{\bt_j}_1.
\end{align*}
\end{proof}

\theoremstyle{plain}
\newtheorem*{theorem1}{Theorem 4.1}
\begin{theorem1}
    Suppose that $h$ is $(\cD, \ell^2)$-individually fair with constant $K_1$, i.e., 
    $$D(h(i, B) , h(j , B)) \leq K_1 || \mathbf{B}_i - \mathbf{B}_j ||_{2},$$ 
    for all $i , j \in [m]$ and $B \in [-1, 1]^{m \times n}$.
    Then, for $f = h \circ \svt(Z , \tau, \psi)$,
	\begin{equation}
		D(f ( i , Z) , f( j , Z) ) 
		\leq K_1 K_2 \norm{\bZ_i - \bZ_j}_1 ,
	\end{equation}
	for all $i ,j \in [m]$, 
    i.e., $f$ is $(D, \ell^1)$-individually fair on $Z$ with constant $K_1 K_2$.
\end{theorem1}

\begin{proof}
Let the singular value decomposition (SVD) of $Z = \sum_{\ell=1}^{\min (m,n)} \sigma_\ell \bu_\ell \bv_\ell^T$,
where $\sigma_i, \bu_i, \bv_i$ are the $i$-th singular value, left singular vector, and right singular vector of $Z$ respectively. Given $f = h \circ \svt(Z , \tau, \psi))$, the input $\hat{A}$ to $h$ is the output of running SVT on $Z$, i.e. 
	\begin{align*}
		\hat{A} = \sum_{\ell \in S(\tau)} \psi( \sigma_\ell ) \bu_\ell \bv_\ell^T .
	\end{align*}
We can expand $||\hat{\bA}_i - \hat{\bA}_j||_2$ to get
\begin{align}
    ||\hat{\bA}_i - \hat{\bA}_j||_2 &= \norm{\sum_{\ell \in S(\tau)}\psi(\sigma_\ell)u_{\ell i}\bv_\ell^T - \sum_{\ell \in S(\tau)}\psi(\sigma_\ell)u_{\ell j}\bv_\ell^T}_2 \\
    &= \norm{\sum_{\ell \in S(\tau)}\psi(\sigma_\ell)(u_{\ell i} - u_{\ell j})\bv_\ell^T}_2.  \label{exp1}
\end{align}
Next we rewrite $u_{\ell i} - u_{\ell j}$ in terms of $\bZ_i$ and $\bZ_j$. Since $\bu_\ell$ is the $\ell$-th left singular vector of $Z$, it is the $\ell$-th eigenvector of $ZZ^T$. Let  $\lambda_\ell$ be the $\ell$-th eigenvalue of $ZZ^T$. Note that $\lambda_\ell = \sigma_\ell^2$. Then 
\begin{equation*}
    \lambda_\ell \bu_\ell = Z Z^T \bu_\ell. 
\end{equation*}
 Looking at only the $i$th row, we see that
\begin{align*}
    \lambda_\ell u_{\ell i} &= \bZ_i Z^T  \bu_\ell \\
    \implies u_{\ell i} &= \frac{\bZ_i Z^T  \bu_\ell}{\lambda_\ell} \\
    \implies u_{\ell i} - u_{\ell j} &= \frac{(\bZ_i - \bZ_j) Z^T  \bu_\ell}{\sigma_\ell^2}.
\end{align*}
Plugging this back into equation \eqref{exp1}, we get 
\begin{align}
    ||\hat{\bA}_i - \hat{\bA}_j||_2 &= \norm{\sum_{\ell \in S(\tau)}\frac{\psi(\sigma_\ell)(\bZ_i - \bZ_j) Z^T  \bu_\ell \bv_\ell^T}{\sigma_\ell^2}}_2 \\ 
    &= \norm{\left(\sum_{\ell \in S(\tau)}\frac{\psi(\sigma_\ell)}{\sigma_\ell^2}\left(\bu_\ell \bv_\ell^T \right)^T \right) Z(\bZ_i - \bZ_j)}_2 \label{mineq}
\end{align}
Next we apply Lemma \ref{lemma:l2} to \eqref{mineq} to get 
\begin{equation}
    ||\hat{\bA}_i - \hat{\bA}_j||_2 \leq \norm{\sum_{\ell \in S(\tau)}\frac{\psi(\sigma_\ell) \bu_\ell \bv_\ell^T}{\sigma_\ell^2}}_\infty \sqrt{n} \norm{Z(\bZ_i - \bZ_j)}_1 \label{l2split}
\end{equation}
Apply Lemma \ref{lemma:l1} to $\norm{Z(\bZ_i - \bZ_j)}_1$ in \eqref{l2split} gives us 
\begin{equation}
    ||\hat{\bA}_i - \hat{\bA}_j||_2 \leq \norm{\sum_{\ell \in S(\tau)}\frac{\psi(\sigma_\ell) \bu_\ell \bv_\ell^T}{\sigma_\ell^2}}_\infty \sqrt{n} \max_k \norm{\bz_k}_1 \norm{\bZ_i - \bZ_j}_1
\end{equation}
Since $D(h(i, B) , h(j , B)) \leq K_1 || \mathbf{B}_i - \mathbf{B}_j ||_{2}$, 
\begin{align}
    D(f(i, Z), f(j, Z)) &= D(h (i , \hat{A}) , h( j , \hat{A}) ) \\
    & \leq K_1 \left( \norm{\sum_{\ell \in S(\tau)}\frac{\psi(\sigma_\ell) \bu_\ell \bv_\ell^T}{\sigma_\ell^2}}_\infty \sqrt{n} \max_k \norm{\bz_k}_1 \right) \norm{\bZ_i - \bZ_j}_1 \\
    &= K_1 K_2 \norm{\bZ_i - \bZ_j}_1.
\end{align}
\end{proof}

\subsection{Proof of Corollary \ref{cor:k1}}

\newtheorem*{corollary2}{Corollary 4.2}
\begin{corollary2}
    Suppose $\psi(x) = \beta x$ and $Z$ satisfies the strong incoherence condition with parameter $\mu_1$ \citep{chen2015incoherence}, i.e., 
    \begin{align*}
        \norm{\sum_{\ell \in S(\tau)} \bu_\ell \bv_\ell^T}_\infty \leq \sqrt{\frac{\mu_1 r}{mn}},
    \end{align*}
    where $r = |S(\tau)|$ denotes the rank of $Z$. 
    Then for any threshold $\tau$, $K_2 \leq {\beta \sqrt{rm}} / {\tau}$.
\end{corollary2}

\begin{proof}
 Recall that 
 \[ K_2 = \norm{\sum_{\ell \in S(\tau)}\frac{\psi(\sigma_\ell) \bu_\ell \bv_\ell^T}{\sigma_\ell^2}}_\infty \sqrt{n} \max_k \norm{\bz_k}_1.\]
 Given $\psi(x) = \beta x$, we have
\begin{equation*}
    K_2 = \beta \norm{\sum_{\ell \in S(\tau)}\frac{\bu_\ell \bv_\ell^T}{\sigma_\ell}}_\infty \sqrt{n} \max_k \norm{\bz_k}_1.
\end{equation*}
Recall $S(\tau) \coloneqq \{\ell: \sigma_\ell > \tau\}$ is the set of components whose singular values exceed $\tau$, so the value of any $\sigma_\ell$ in the denominator must be at least $\tau$, giving us
\begin{equation*}
    K_2 \leq  \frac{\beta}{\tau} \norm{\sum_{\ell \in S(\tau)}\bu_\ell \bv_\ell}_\infty \sqrt{n} \max_k \norm{\bz_k}_1.
\end{equation*}
Given $Z$ satisfies the strong incoherence condition, 
\begin{equation*}
    K_2 \leq  \frac{\beta}{\tau} \cdot \sqrt{\frac{r}{mn}} \cdot \sqrt{n} \max_k \norm{\bz_k}_1.
\end{equation*}
Since each entry $Z_{ij} \in [-1,1]$ and there are $m$ entries in each column of $Z$, $\norm{\bz_k}_1 \leq m$.  Hence
\begin{equation*}
    K_2 \leq \frac{\beta}{\tau} \cdot \sqrt{\frac{r}{mn}} \cdot \sqrt{n} \cdot m \leq \frac{\beta \sqrt{rm}}{\tau},
\end{equation*}
concluding our proof.
\end{proof}

\subsection{Proof of Theorem \ref{thm:IF_a}}

\newtheorem*{theorem3}{Theorem 4.3}
\begin{theorem3} 
    Let $d$ denote the $\ell^q$ norm. 
	Suppose that $h$ is $(\cD, d)$-individually fair with constant $K_1$, i.e., 
    $$D(h(i, B ) , h(j , B )) \leq K_1 || \mathbf{B}_i - \mathbf{B}_j ||_{q},$$ 
    for all $i , j \in [m]$ and $B \in [-1, 1]^{m \times n}$.
     Then, for $f = h \circ \svt(Z , \tau, \psi)$,
	\begin{align}
		D(f (i, Z & ) ,  f( j , Z) )  
		\leq  K_1 \norm{\bA_i - \bA_j}_q 
  + 2 K_1 || \hat{A} - A ||_{q,\infty}
	\end{align} 
    for all $i ,j \in [m]$.
\end{theorem3}

\begin{proof}
	Recall that $\norm{M}_{q, \infty} = \max_{i} \norm{ \mathbf{m}_i }_q$.  
	This result follows from the application of the triangle inequality. 
	\begin{align*}
		D(f (i, Z ) , f( j , Z) )  
			&= D( h ( i , \hat{A}) , h ( j , \hat{A} ) )
			\\
			&\leq K_1 \norm{ \hat{\bA}_i - \hat{\bA}_j }_q
			\\
			&\leq K_1 (|| \hat{\bA}_i - \bA_i ||_q
				+ || \hat{\bA}_j - \bA_j ||_q
				+ || \bA_i - \bA_j ||_q )
			\\
			&\leq K_1 ( 2 || \hat{A} - A ||_ {q, \infty}
			+ || \bA_i - \bA_j ||_q )
			\\
			&\leq K_1 || \bA_i - \bA_j ||_q + 2 K_1 || \hat{A} - A ||_{q,\infty} ,
	\end{align*}
	which gives the result as stated. 
\end{proof}

\subsection{Modification of Theorem 1.1 in Chatterjee (2015)}
\newtheorem*{theorem11}{Theorem 1.1 from \citet*{chatterjee}}
\begin{theorem11}
Suppose that we have a $m \times n$ matrix $M$, where $m \leq n$ and the entries of $M$ are bounded by 1 in absolute value. Let $X$ be a matrix whose elements are independent random variables, and $\mathbb{E}(x_{ij}) = m_{ij}$ for all
$i$ and $j$. Assume that the entries of X are also bounded by $1$ in absolute
value, with probability one. Let $p$ be a real number belonging to the interval $[0,1]$. Suppose that each entry of $X$ is observed with probability $p$, and unobserved with probability $1-p$, independently of the other entries.
 
 We construct an estimator $\hat{M}$ of $M$ based on the observed entries of $X$ using the Universal Singular Value Thresholding (USVT) algorithm with threshold $(2 + \eta)\sqrt{n\hat{p}}.$ Suppose that $p \geq n^{-1 + \varepsilon}$ for some $\varepsilon > 0.$ Then 
\[ \text{MSE}(\hat{M}) \leq C \min \left(\frac{\norm{M}_*}{m\sqrt{np}},\frac{\norm{M}_*^2}{mn}, 1\right) + C(\varepsilon)e^{-cnp},\]
where $C$ and $c$ are positive constants that depend only on the choice of $\eta$ and $C(\varepsilon)$ depends only on $\varepsilon$ and $\eta.$ 
\end{theorem11}

In our work, we have modified Theorem 1.1 from \citet*{chatterjee} for our specific setup. The modifications only involve the renaming of variables to keep our notation consistent and to clarify the dependencies between variables. The changes are summarized in the following table.

\begin{table}[h]
\begin{center}
\caption{Modifications to notation in Theorem 1.1 of \citet{chatterjee}.}
\vskip 0.1in
\begin{sc}
\begin{small}
\begin{tabular}{c|c}
\toprule
\textbf{Notation in \citet{chatterjee}} & \textbf{Our Notation} \\
\noalign{\vskip 0.5ex} 
\hline 
\noalign{\vskip 0.5ex} 
$M$ & $A$ \\
$X$ & $Z$ \\
$\hat{M}$ & $\text{SVT}(M, \tau, \psi)$ \\
$n$ & $\rho_1$ \\
$m$ & $\rho_2$ \\
$C$ & $C(\eta)$ \\ 
$c$ & $c(\eta)$ \\
$C(\varepsilon)$ & $C(\varepsilon, \eta)$ \\
\bottomrule
\end{tabular}
\end{small}
\end{sc}
\end{center}
\end{table}

Our modified proposition is as follows.

\newtheorem*{theorem45}{Proposition \ref{prop:USVT_MSE}.}
\begin{theorem45}[Modified from Theorem 1.1. in \citet*{chatterjee}]
	Suppose the elements of $Z$ are independent random variables, each independently observed with probability $p \in [0,1]$.
	Let $\hat{p}$ be the proportion of observed values, $\psi( x) = x / \hat{p}$, $\epsilon \in (0,1]$, and $w = (2 + \eta)^2$ for $\eta \in (0,1)$. 
	Let $\rho_1 = \max(m,n)$ and $\rho_2 = \min(m,n)$.
	Then, if $p \geq \rho_1^{\epsilon - 1}$ for some $\epsilon > 0$ and $\tau = \sqrt{w \rho_1 \hat{p}}$, 
	\begin{align*}
		\mse \left( \svt(Z, \tau, \psi) \right) \leq C(\eta) \min \left(
		\frac{\norm{A}_*}{\rho_2 \sqrt{\rho_1 p}} , 
		\frac{\norm{A}_*^2}{ \rho_1 \rho_2 } , 
		1
		\right) + C(\epsilon, \eta) \exp( - c (\eta) \rho_1 p ) ,
	\end{align*} 
	where $C(\eta) , c(\eta) > 0$ depend only on $\eta$
	and $C(\epsilon, \eta)$ depends only on $\eta$ and $\epsilon$.
\end{theorem45}

\subsection{Experimental Setup}
Below are some additional details about our experimental setup. \\ 

\noindent \textbf{Training the DNN. } Given input matrix $B$, the training set of the deep neural net (DNN) described in Section \ref{sec:synth-setup} consists of (input, target) tuples of the form $(\begin{bmatrix} \mathbf{B}i & \mathbf{b}_j \end{bmatrix}, B_{ij})$. Missing entries in $\mathbf{B}_i$ and $\mathbf{b}_j$ are replaced with zeros. Out of the observed entries $(i,j) \in \Omega $, 80\% are used for training and the remaining 20\% are used for validation; the unobserved entries form our test set. We use a batch size of 128 and 2000 steps of training.

\subsection{Additional Experimental Results}

\subsubsection{Experiment \#3: Observing entries non-uniformly at random} \textbf{Setup. } Recall the setup for Experiment \#1 in Section \ref{sec:synth-setup}. We sample $m = 200$ feature vectors of length $n = 800$, each corresponding to an individual, to form the ground truth matrix $A \in [-1,1]^{m\times n}.$ The feature vectors are sampled from $c = 10$ clusters, where each cluster is a multivariate normal distribution. 

Next we generate the observation matrix $Z$. Recall that $\Omega$ denotes the set of observed entries. Instead of selecting each entry independently with probability $p$, we instead observe entries with different probabilities depending on the cluster it belongs to. For each cluster $k$,  there is an associated random vector $\bp_k \in \mathbb{R}^n$ with entries summing to $p \cdot n$. For each individual $i$ in cluster $k$, the entry $(i,j)$ is observed with probability $p_i[j].$ The expected number of observed entries is $p \cdot n \cdot m$, so the proportion observed is as desired, but the entries are no longer drawn uniformly at random as the probability an entry is drawn is dependent on the cluster it is in. The remaining setup is identical to that for Experiment \#1. \\  

\noindent \textbf{Results. } Table \ref{synth-group} summarizes the results for Experiment \#3. The values are averaged over 10 simulations, and the error bars give $+/-$ two standard deviations. We observe that $\IF^f_1(Z)$ is much smaller than $\IF^h_1(Z)$, which again verifies SVT pre-processing improves IF on $Z$.

Note that the entries $(i,j) \in \Omega$ not being selected uniformly at random violates one of the conditions of Proposition \ref{prop:USVT_MSE}, which states that each entry of $A$ is independently observed with probability $p$. 
Despite violating this condition, we observe in Table \ref{synth-group} that there is minimal decrease in performance when applying SVT pre-processing. This indicates the performance guarantees of SVT are robust to relaxations of the independence condition stated in Proposition \ref{prop:USVT_MSE}. 

\begin{table}[t]
\begin{center}
\caption{Results on IF and performance in Experiment \#3.}
\label{synth-group}
\vskip 0.1in
\begin{small}
\begin{sc}
\begin{tabular}{lcccr}
\toprule
{} & {$\hat{p} = 0.05$} & {$\hat{p} = 0.1$} & {$\hat{p} = 0.2$} & {$\hat{p} = 0.4$} \\
\midrule
    {MSE($h$)}  & {$0.31 \pm 0.002$} & {$0.23 \pm 0.002$} & {$0.16 \pm 0.002$} & {$0.14 \pm 0.001$} \\
\midrule
    {MSE($f$)}  & {$0.33 \pm 0.003$} & {$0.23 \pm 0.001$} & {$0.17 \pm 0.001$} & {$0.14 \pm 0.001$} \\
\midrule
    {$\IF^h_1(Z)$}  & {$0.24 \pm 0.005$} & {$0.17 \pm 0.003$} & {$0.11 \pm 0.002$} & {$0.06 \pm 0.001$} \\
\midrule
    {$\IF^f_1(Z)$}  & {$0.02 \pm 0.001$} & {$0.02 \pm 0.001$} & {$0.03 \pm 0.001$} & {$0.03 \pm 0.001$} \\
\midrule 
    {\centering $K_2$}  & {$0.05 \pm 0.003$} & {$0.11 \pm 0.007$} & {$0.25 \pm 0.009$} & {$0.49 \pm 0.013$} \\
\midrule
    {$\IF^h_2(A)$}  & {$0.46 \pm 0.010$} & {$0.63 \pm 0.011$} & {$0.75 \pm 0.015$} & {$0.75 \pm 0.007$} \\
\midrule
    {$\IF^f_2(A)$}  & {$0.49 \pm 0.014$} & {$0.64 \pm 0.006$} & {$0.75 \pm 0.008$} & {$0.73 \pm 0.014$} \\
\bottomrule
\end{tabular}
\end{sc}
\end{small}
\end{center}
\end{table}

\subsubsection{Experiment \#4: Varying length of feature vectors} 
\textbf{Setup. } Consider $m = 500$ individuals. We sample $m$ feature vectors of length $n$ from $c = 20$ clusters, where each cluster is a multivariate normal distribution. The mean of each cluster is a vector of length $n$ drawn uniformly at random from $(-1, 1)$, and the covariance of each cluster is an $n \times n$ diagonal matrix with whose diagonal values are sampled uniformly at random from $(0, 0.1)$.
The feature vectors are then clipped so that all values fall within $[-1, 1]$. When generating the observation matrix $Z$, we observe a proportion $p = 0.2$ of entries uniformly at random. 
Instead of varying the value of $p$, we instead create datasets for varying values of $n$, the length of the feature vector.
The remaining setup is identical to that for Experiment 1, described in Section \ref{sec:synth-setup}.

\begin{table}[h]
\caption{Results on IF and performance over different values of $n$ in Experiment \#4. }
\label{synth-n}
\vskip 0.1in
\begin{center}
\begin{small}
\begin{sc}
\begin{tabular}{lccccr}
\toprule
{} & {$n = 25$} & {$ n = 100 $} & {$n = 400$} & {$n = 800$} \\
\midrule
    {MSE($h$)}  & {$0.36 \pm 0.004$} & {$0.28 \pm 0.003$} & {$0.13 \pm 0.001$} & {$0.10 \pm 0.001$} \\
\midrule
    {MSE($f$)}  & {$0.36 \pm 0.004$} & {$0.28 \pm 0.003$} & {$0.12 \pm 0.001$} & {$0.10 \pm 0.001$} \\
\midrule
    {$\IF^h_1(Z)$}  & {$0.43 \pm 0.004$} & {$0.26 \pm 0.002$} & {$0.17 \pm 0.001$} & {$0.13 \pm 0.001$} \\
\midrule
    {$\IF^f_1(Z)$}  & {$0.15 \pm 0.003$} & {$0.07 \pm 0.001$} & {$0.05 \pm 0.001$} & {$0.03 \pm 0.001$} \\
\midrule 
    {\centering $K_2$}  & {$0.21 \pm 0.020$} & {$0.06 \pm 0.002$} & {$0.03 \pm 0.001$} & {$0.02 \pm 0.001$} \\
\midrule
    {$\IF^h_2(A)$}  & {$0.51 \pm 0.004$} & {$0.63 \pm 0.005$} & {$0.80 \pm 0.007$} & {$0.83 \pm 0.007$} \\
\midrule
    {$\IF^f_2(A)$}  & {$0.62 \pm 0.011$} & {$0.67 \pm 0.007$} & {$0.82 \pm 0.011$} & {$0.83 \pm 0.010$} \\
\bottomrule
\end{tabular}
\end{sc}
\end{small}
\end{center}
\end{table}

\noindent \textbf{Results. } Table \ref{synth-n} summarizes the results for Experiment \#4. Recall from Section \ref{sec:discussion} that our theoretical guarantees for SVT pre-processing simultaneously strengthening IF and having good prediction performance hold when $n = \omega(rm/\hat{p}).$ In the above setting, both $n=25$ and $n = 100$ fall within $o(rm/\hat{p}).$ However, we observe that $\IF^f_1(Z)$ is much smaller than $\IF^h_1(Z)$ for all the values of $n$, and there are very minimal differences between MSE$(h)$ and MSE$(f)$. This means that even when $n = o(rm/\hat{p})$, we still see large improvements in IF with respect to $Z$ with little to no effect on prediction performance when applying SVT pre-processing. This demonstrates that the empirical IF and performance benefits of SVT pre-processing are not restricted to when $n = \omega(rm/\hat{p}).$

\end{document}